\documentclass[american]{llncs}
\usepackage{todonotes}
\usepackage{babel}
\usetikzlibrary{shapes,shapes.geometric,backgrounds,calc,fit,positioning}

\usepackage[utf8]{inputenc}
\usepackage{graphicx}

\usepackage[ruled,vlined]{algorithm2e}

\usepackage{amsmath,amssymb}
\usepackage{mathtools}
\usepackage{commath}
\usepackage{faktor} 
\usepackage{ntheorem}
\usepackage{nicefrac}
\usepackage{enumerate}
\usepackage{paralist}

\usepackage[hidelinks]{hyperref}
\usepackage[all]{hypcap}

\usepackage[noabbrev]{cleveref}

\usepackage[backend=biber,style=numeric-comp,url=false,firstinits=true,
doi=false]{biblatex}
\usepackage{booktabs}
\usepackage{arydshln}

\title{Exploring Scale-Measures of Data Sets}

\date{tba 2021}

\institute{University of Kassel, Germany}

\usepackage{fca, newdrawline}

\crefalias{problem}{prob}

\newcommand{\Scon}{\mathbb{S}}
\newcommand{\Tcon}{\mathbb{T}}

\newcommand{\Sh}{\mathfrak{S}}
\newcommand{\SH}{\underline{\mathfrak{S}}}

\DeclareMathOperator{\id}{id}

\newcommand*{\logeq}{\ratio\Leftrightarrow}

\DeclareMathOperator{\Ext}{Ext}
\DeclareMathOperator{\Int}{Int}

\DeclareMathOperator{\app}{\mid}

\usepackage{etoolbox,refcount}
\usepackage{multicol}
\newcounter{countitems}
\newcounter{nextitemizecount}
\newcommand{\setupcountitems}{%
  \stepcounter{nextitemizecount}%
  \setcounter{countitems}{0}%
  \preto\item{\stepcounter{countitems}}%
}
\makeatletter
\newcommand{\computecountitems}{%
  \edef\@currentlabel{\number\c@countitems}%
  \label{countitems@\number\numexpr\value{nextitemizecount}-1\relax}%
}
\newcommand{\nextitemizecount}{%
  \getrefnumber{countitems@\number\c@nextitemizecount}%
}
\newcommand{\previtemizecount}{%
  \getrefnumber{countitems@\number\numexpr\value{nextitemizecount}-1\relax}%
}
\makeatother    
    {\end{itemize}%
    \unskip\computecountitems\ifnumcomp{\previtemizecount}{>}{3}{\end{multicols}}{}}

\let\cref\Cref

\begin{document}

\author{Tom Hanika\inst{1,2}, Johannes Hirth\inst{1,2}}
\institute{%
  Knowledge \& Data Engineering Group,
  University of Kassel, Germany\\[0.5ex]
  \and
  Interdisciplinary Research Center for Information System Design\\
  University of Kassel, Germany\\[0.5ex]
  \email{tom.hanika@cs.uni-kassel.de, hirth@cs.uni-kassel.de}
}

\maketitle

\begin{abstract}
  Measurement is a fundamental building block of numerous scientific
  models and their creation. This is in particular true for data
  driven science. Due to the high complexity and size of modern data
  sets, the necessity for the development of understandable and
  efficient scaling methods is at hand. A profound theory for scaling
  data is scale-measures, as developed in the field of formal concept
  analysis. Recent developments indicate that the set of all
  scale-measures for a given data set constitutes a lattice and does
  hence allow efficient exploring algorithms. In this work we study
  the properties of said lattice and propose a novel scale-measure
  exploration algorithm that is based on the well-known and proven
  attribute exploration approach. Our results motivate multiple
  applications in scale recommendation, most prominently
  (semi-)automatic scaling.

  \textbf{keywords}: FCA, Conceptual~Measures, Data~Scaling,
  Measurements, Formal~Concept, Lattice

\end{abstract}

\section{Introduction}
An inevitable step of any data-based knowledge discovery process is
\emph{measurement}~\cite{pfanzagl1971theory} and the associated (explicit or
implicit) \emph{scaling} of the data~\cite{stevens1946theory}. The latter is
particularly constrained by the underlying mathematical formulation of
the data representation, e.g., real-valued vector spaces or weighted
graphs, the requirements of the data procedures, e.g., the presence of
a distance function, and, more recently, the need for human
understanding of the results.
Considering the scaling of data as part of the analysis itself, in
particular formalizing it and thus making it controllable, is a
salient feature of formal concept analysis (FCA)~\cite{fca-book}.
This field of research has spawned a variety of specialized scaling
methods, such as logical scaling~\cite{logiscale}, and in the form of
\emph{scale-measures} links the scaling process with the study of
\emph{continuous mappings} between \emph{closure systems}. 

Recent results by the authors~\cite{navimeasure} revealed that the set
of all scale-measures for a given data set constitutes a
lattice. Furthermore, it was shown that any scale-measure can be
expressed in simple propositional terms using disjunction, conjunction
and negation. Among other things, the previous results allow a
computational transition between different scale-measures, which we may call
\emph{scale-measure navigation}, as well as their
\emph{interpretability} by humans.

Despite these advances, the question of how to identify appropriate
and meaningful scale-measures for a given data set with respect to a
human data analyst and how to express that meaningfulness in the first
place remains unanswered.  In this paper, we propose an answer to this
question by adapting the well-known \emph{attribute exploration
  algorithm} from FCA to present a method for exploring scale
measures. Very similar to the original algorithm does
\emph{scale-measure exploration} inquire a (human) scaling expert for
how to aggregate, separate, omit, or introduce data set features. Our
efforts do finally result in a (semi-)automatic scaling framework
which may be applied to large and complex data sets.

In detail, after recalling scale-measure basics in~\cref{ideals} we
apply theoretical results for ideals in closure systems to the lattice
of all scale-measures. From this we derive notions for the relevance
of scale-measures as well as the mentioned novel exploration method
in~\cref{sec4}, which is supported by a detailed example. Finally,
in~\cref{sec:5}, we outline the (semi-)automatic scaling framework and
conclude in~\cref{conclusion} after revisiting related work about
scaling in~\cref{sec:relate}.

\section{Scales and Measurement}
\subsubsection*{FCA Recap}

Formalizing and understanding the process of \emph{measurement} is, in
particular in data science, an ongoing discussion, for which we refer
the reader to \emph{Representational Theory of
  Measurement}~\cite{suppes1989foundations,luce1990foundations} as
well as \emph{Numerical Relational
  Structure}~\cite{pfanzagl1971theory}, and \emph{algebraic
  (measurement) structures}~\cite[p. 253]{roberts1984measurement}.

Formal concept analysis (FCA)~\cite{Wille1982, fca-book} is well
equipped to handle and comprehend data scaling tasks. In FCA the basic
data structure is the \emph{formal contexts} as seen in the
example~\cref{fig:bj1} (top), i.e., a triple $(G,M,I)$ with non-empty and finite set
$G$ (called \emph{objects}), non-empty and finite set $M$ (called \emph{attributes})
and a binary relation $I \subseteq G \times M$ (called
\emph{incidence}). We say $(g,m) \in I$ is equivalent to ``$g$ has
attribute $m$''.  We call $\Scon=(H,N,J)$ an \emph{induced
sub-context} of $\context$, iff $H\subseteq G, N\subseteq M$ and
$I_\Scon=I\cap (H_\Scon \times N)$, and write $\Scon \leq
\context$. We find two operators $\cdot':\mathcal{P}(G)\to
\mathcal{P}(M), A\mapsto A'=\{m\in M \mid \forall a \in A:(a,m)\in
I\}$, and $\cdot':\mathcal{P}(M)\to\mathcal{P}(G), B\mapsto B'=\{g\in
G\mid \forall b\in B:(g,b)\in I\}$, called \emph{derivations}. Pairs
$(A,B) \in \mathcal{P}(G)\times \mathcal{P}(M)$ with $A'=B$ and $A =
B'$, are called \emph{formal concepts}, where $A$ is called
\emph{extent} and $B$ \emph{intent}. Consecutive application leads to
two \emph{closure spaces} $\Ext(\context)\coloneqq (G,'')$ and
$\Int(\context)\coloneqq(M,'')$. Both closure systems are represented
in the \emph{(concept) lattice}
$\BV(\context)=(\mathcal{B}(\context),\subseteq)$, where
$\mathcal{B}(\context)\coloneqq
\{(A,B)\in\mathcal{P}(G)\times\mathcal{P}(M)\mid A'=B\wedge B'=A\}$ is
the set of concepts of $\context$ and the order relation is $(A,B)\leq
(C,D)\logeq A\subseteq C$.

\begin{figure}[t]
  \label{bjice}
  \centering
    \scalebox{0.55}{
      \hspace{-1.4cm}
      \begin{cxt}
        \cxtName{}
        \att{\shortstack{has limbs\\ (L) }}
        \att{\shortstack{breast feeds\\ (BF) }}
        \att{\shortstack{needs\\ chlorophyll (Ch)}}
        \att{\shortstack{needs water\\ to live (W)}}
        \att{\shortstack{lives on\\ land (LL)}}
        \att{\shortstack{lives in\\ water (LW)}}
        \att{\shortstack{can move\\ \ (M)}}
        \att{\shortstack{monocotyledon\\ (MC) }}
        \att{\shortstack{dicotyledon\\ (DC) }}
        \obj{xx.xx.x..}{\shortstack{dog\\ \ }}
        \obj{...x.xx..}{\shortstack{fish\\ leech }}
        \obj{..xxx..x.}{\shortstack{corn\\ \ }}
        \obj{x..x.xx..}{\shortstack{bream\\ \ }}
        \obj{..xx.x.x.}{\shortstack{water\\ weeds}}
        \obj{..xxx...x}{\shortstack{bean\\ \ }}
        \obj{x..xxxx..}{\shortstack{frog\\ \ }}
        \obj{..xxxx.x.}{\shortstack{reed\\ \ }}
      \end{cxt}}  

    \scalebox{0.4}{\colorlet{mivertexcolor}{black!80}
\colorlet{jivertexcolor}{black!80}
\colorlet{vertexcolor}{black!80}
\colorlet{bordercolor}{black!80}
\colorlet{linecolor}{gray}
\tikzset{vertexbase/.style={semithick, shape=circle, inner sep=2pt, outer sep=0pt, draw=bordercolor},%
  vertex/.style={vertexbase, fill=vertexcolor!45},%
  mivertex/.style={vertexbase, fill=mivertexcolor!45},%
  jivertex/.style={vertexbase, fill=jivertexcolor!45},%
  divertex/.style={vertexbase, top color=mivertexcolor!45, bottom color=jivertexcolor!45},%
  conn/.style={-, thick, color=linecolor}%
}
\begin{tikzpicture}[scale=0.35,font=\footnotesize]
  \begin{scope} 
    \begin{scope} 
      \foreach \nodename/\nodetype/\xpos/\ypos in {%
        0/vertex/0.0/0.0,
        1/jivertex/-6.0/10.0,
        2/jivertex/6.0/10.0,
        3/divertex/-12.0/14.0,
        4/jivertex/-6.0/14.0,
        5/jivertex/5.0/15.0,
        6/jivertex/9.0/15.0,
        7/divertex/13.0/15.0,
        8/vertex/-11.0/17.0,
        9/jivertex/-5.0/17.0,
        10/vertex/0.0/18.0,
        11/mivertex/7.0/19.0,
        12/vertex/11.0/19.0,
        13/mivertex/-11.0/21.0,
        14/mivertex/-1.0/23.0,
        15/mivertex/-10.0/24.0,
        16/mivertex/10.0/24.0,
        17/mivertex/1.0/31.0,
        18/vertex/0.0/36.0
      } \node[\nodetype] (\nodename) at (\xpos, \ypos) {};
    \end{scope}
    \begin{scope} 
      \path (2) edge[conn] (6);
      \path (1) edge[conn] (4);
      \path (12) edge[conn] (17);
      \path (6) edge[conn] (11);
      \path (14) edge[conn] (18);
      \path (0) edge[conn] (2);
      \path (9) edge[conn] (14);
      \path (2) edge[conn] (5);
      \path (10) edge[conn] (17);
      \path (7) edge[conn] (12);
      \path (0) edge[conn] (1);
      \path (6) edge[conn] (12);
      \path (15) edge[conn] (18);
      \path (9) edge[conn] (15);
      \path (5) edge[conn] (14);
      \path (17) edge[conn] (18);
      \path (1) edge[conn] (10);
      \path (0) edge[conn] (7);
      \path (8) edge[conn] (17);
      \path (13) edge[conn] (15);
      \path (4) edge[conn] (13);
      \path (0) edge[conn] (3);
      \path (8) edge[conn] (13);
      \path (4) edge[conn] (9);
      \path (5) edge[conn] (11);
      \path (12) edge[conn] (16);
      \path (2) edge[conn] (10);
      \path (3) edge[conn] (8);
      \path (16) edge[conn] (18);
      \path (11) edge[conn] (16);
      \path (10) edge[conn] (14);
      \path (1) edge[conn] (8);
    \end{scope}
    \begin{scope} 
      \foreach \nodename/\labelpos/\labelopts/\labelcontent in {%
        1/below//{frog},
        2/below//{reed},
        3/below//{dog},
        3/above//{BF},
        4/below//{bream},
        5/below//{water weeds},
        6/below//{corn},
        7/below//{bean},
        7/above//{DC},
        9/below//{fish leech},
        11/above//{MC},
        13/above//{L},
        14/above//{LW},
        15/above//{M},
        16/above//{Ch},
        17/above//{LL},
        18/above//{W}
      } \coordinate[label={[\labelopts]\labelpos:{\labelcontent}}](c) at (\nodename);
    \end{scope}
  \end{scope}
\end{tikzpicture}}
  \caption{This Figure shows the \emph{Living Beings and Water}
    context in the top. Its concept lattice is displayed at the bottom and
    contains nineteen concepts.}
  \label{fig:bj1}
\end{figure}
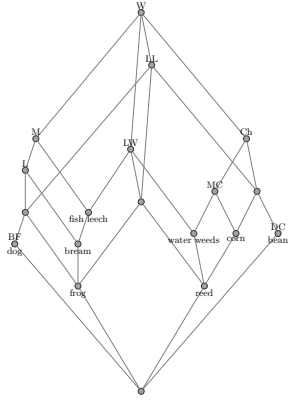

\subsection{Scales-Measures}\label{sec:motivate}
A fundamental approach to comprehensible scaling, in particular for nominal and
ordinal data as studied in this work, is the following. 

\begin{definition}[Scale-Measure (cf. Definition 91, \cite{fca-book})]
\label{def:sm}
Let $\context = (G,M,I)$ and
$\mathbb{S}=(G_{\mathbb{S}},M_{\mathbb{S}},I_{\mathbb{S}})$ be a
formal contexts. The map $\sigma :G \rightarrow G_{\mathbb{S}}$ is
called an \emph{$\mathbb{S}$-measure of $\context$ into the scale
  $\mathbb{S}$} iff the preimage
$\sigma^{-1}(A)\coloneqq \{g\in G\mid \sigma(g)\in A\}$
of every extent $A\in \Ext(\Scon)$ is an extent of $\context$.
\end{definition}

This definition resembles the idea of \emph{continuity between closure
  spaces} $(G_1,c_1)$ and $(G_2,c_2)$. We say that the map $f:G_1\to
G_2$ is \emph{continuous} if and only if $\text{for all}\
A\in\mathcal{P}(G_2) \text{ we have } c_1(f^{-1}(A))\subseteq
f^{-1}(c_2(A))$. This property is equivalent to the requirement
in~\cref{def:sm} that the preimage of closed sets is closed.

In the light of the defnition above we understand $\sigma$ as an
interpretation of the objects from $\context$ in $\Scon$. Therfore we
view the set
$\sigma^{-1}(\Ext(\Scon))\coloneqq\bigcup_{A\in\Ext(\Scon)}\sigma^{-1}(A)$
as the set of extents that is \emph{reflected} by the scale context
$\Scon$.

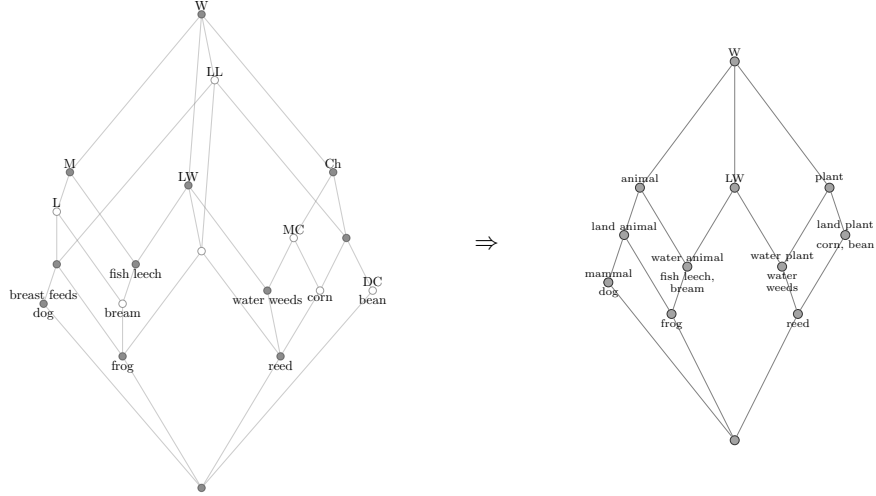
\begin{figure}[t]
  \label{bjicemeasure}
\hspace{-1cm}  \begin{minipage}{.71\linewidth}
    \scalebox{0.55}{
      \begin{cxt}
        \cxtName{}
        \att{\shortstack{W}}
        \att{\shortstack{LW}}
        \att{\shortstack{plants}} 
        \att{\shortstack{animals}} 
        \att{\shortstack{land plants}} 
        \att{\shortstack{water plants}} 
        \att{\shortstack{land animal}} 
        \att{\shortstack{water animal}} 
        \att{\shortstack{mammal}} 
        \obj{x..x..x.x}{\shortstack{dog\\ \ }}
        \obj{xx.x...x.}{\shortstack{fish\\ leech }}
        \obj{x.x.x....}{\shortstack{corn\\ \ }}
        \obj{xx.x...x.}{\shortstack{bream\\ \ }}
        \obj{xxx..x...}{\shortstack{water\\ weeds}}
        \obj{x.x.x....}{\shortstack{bean\\ \ }}
        \obj{xx.x..xx.}{\shortstack{frog\\ \ }}
        \obj{xxx.xx...}{\shortstack{reed\\ \ }}
      \end{cxt}}  
    \end{minipage}
    \begin{minipage}{.36\linewidth}
  \scalebox{0.8}{\vbox{ \begin{itemize}[]
      \item plants $\coloneqq$ Ch
      \item[] animals $\coloneqq$ M  
      \item[] land plants $\coloneqq$ LL $\wedge$ plant 
      \item[] water plants $\coloneqq$ LW $\wedge$ plant 
      \item[] land animal $\coloneqq$ LL $\wedge$ animal
      \item[] water animal $\coloneqq$ LW $\wedge$ animal
      \item[] mammal $\coloneqq$ animal $\wedge$ BF
      \end{itemize}}}
    \end{minipage}
  \begin{tikzpicture}
    \node at (0,0)
    {\scalebox{0.5}{\colorlet{mivertexcolor}{white}
\colorlet{jivertexcolor}{black}
\colorlet{vertexcolor}{black}
\colorlet{bordercolor}{black}
\colorlet{linecolor}{gray}
\tikzset{vertexbase/.style={semithick, shape=circle, inner sep=2pt, outer sep=0pt, draw=bordercolor,draw opacity=0.4},%
  vertex/.style={vertexbase, fill=vertexcolor!45},%
  mivertex/.style={vertexbase, fill=mivertexcolor!45},%
  jivertex/.style={vertexbase, fill=jivertexcolor!45},%
  divertex/.style={vertexbase, top color=mivertexcolor!45, bottom color=jivertexcolor!45},%
  conn/.style={-, thick, color=linecolor}%
}
\tikzstyle{n} = [text width=2.5cm,align=center]
\begin{tikzpicture}[scale=0.35,font=\normalsize]
  \begin{scope} 
    \begin{scope} 
      \foreach \nodename/\nodetype/\xpos/\ypos in {%
        0/vertex/0.0/0.0,
        1/vertex/-6.0/10.0,
        2/vertex/6.0/10.0,
        3/vertex/-12.0/14.0,
        4/mivertex/-6.0/14.0,
        5/vertex/5.0/15.0,
        6/mivertex/9.0/15.0,
        7/mivertex/13.0/15.0,
        8/vertex/-11.0/17.0,
        9/vertex/-5.0/17.0,
        10/mivertex/0.0/18.0,
        11/mivertex/7.0/19.0,
        12/vertex/11.0/19.0,
        13/mivertex/-11.0/21.0,
        14/vertex/-1.0/23.0,
        15/vertex/-10.0/24.0,
        16/vertex/10.0/24.0,
        17/mivertex/1.0/31.0,
        18/vertex/0.0/36.0
      } \node[\nodetype] (\nodename) at (\xpos, \ypos) {};
    \end{scope}
    \begin{scope} 
      \path (2) edge[conn,draw opacity=0.4] (6);
      \path (1) edge[conn,draw opacity=0.4] (4);
      \path (12) edge[conn,draw opacity=0.4] (17);
      \path (6) edge[conn,draw opacity=0.4] (11);
      \path (14) edge[conn,draw opacity=0.4] (18);
      \path (0) edge[conn,draw opacity=0.4] (2);
      \path (9) edge[conn,draw opacity=0.4] (14);
      \path (2) edge[conn,draw opacity=0.4] (5);
      \path (10) edge[conn,draw opacity=0.4] (17);
      \path (7) edge[conn,draw opacity=0.4] (12);
      \path (0) edge[conn,draw opacity=0.4] (1);
      \path (6) edge[conn,draw opacity=0.4] (12);
      \path (15) edge[conn,draw opacity=0.4] (18);
      \path (9) edge[conn,draw opacity=0.4] (15);
      \path (5) edge[conn,draw opacity=0.4] (14);
      \path (17) edge[conn,draw opacity=0.4] (18);
      \path (1) edge[conn,draw opacity=0.4] (10);
      \path (0) edge[conn,draw opacity=0.4] (7);
      \path (8) edge[conn,draw opacity=0.4] (17);
      \path (13) edge[conn,draw opacity=0.4] (15);
      \path (4) edge[conn,draw opacity=0.4] (13);
      \path (0) edge[conn,draw opacity=0.4] (3);
      \path (8) edge[conn,draw opacity=0.4] (13);
      \path (4) edge[conn,draw opacity=0.4] (9);
      \path (5) edge[conn,draw opacity=0.4] (11);
      \path (12) edge[conn,draw opacity=0.4] (16);
      \path (2) edge[conn,draw opacity=0.4] (10);
      \path (3) edge[conn,draw opacity=0.4] (8);
      \path (16) edge[conn,draw opacity=0.4] (18);
      \path (11) edge[conn,draw opacity=0.4] (16);
      \path (10) edge[conn,draw opacity=0.4] (14);
      \path (1) edge[conn,draw opacity=0.4] (8);
    \end{scope}
    \begin{scope} 
      \foreach \nodename/\labelpos/\labelopts/\labelcontent in {%
        1/below//{frog},
        2/below//{reed},
        3/below//{dog},
        3/above//{breast feeds},
        4/below//{bream},
        5/below//{water weeds},
        6/below//{corn},
        7/below//{bean},
        7/above//{DC},
        9/below//{fish leech},
        11/above//{MC},
        13/above//{L},
        14/above//{LW},
        15/above//{M},
        16/above//{Ch},
        17/above//{LL},
        18/above//{W}
      } \coordinate[label={[\labelopts]\labelpos:{\labelcontent}}](c) at (\nodename);
    \end{scope}
  \end{scope}
\end{tikzpicture}}};
    \node[draw opacity = 0,draw=white, text=black] at (4,0)
    {$\Rightarrow$};
    \node at (7,0)
    {\scalebox{0.6}{\colorlet{mivertexcolor}{black!80}
\colorlet{jivertexcolor}{black!80}
\colorlet{vertexcolor}{black!80}
\colorlet{bordercolor}{black!80}
\colorlet{linecolor}{gray}
\tikzset{vertexbase/.style={semithick, shape=circle, inner sep=2pt, outer sep=0pt, draw=bordercolor},%
  vertex/.style={vertexbase, fill=vertexcolor!45},%
  mivertex/.style={vertexbase, fill=mivertexcolor!45},%
  jivertex/.style={vertexbase, fill=jivertexcolor!45},%
  divertex/.style={vertexbase, top color=mivertexcolor!45, bottom color=jivertexcolor!45},%
  conn/.style={-, color=linecolor}%
}
\tikzstyle{n} = [text width=1.3cm,align=center]
\begin{tikzpicture}[scale=0.35,font=\scriptsize]
  \begin{scope} 
    \begin{scope} 
      \foreach \nodename/\nodetype/\xpos/\ypos in {%
        0/vertex/0.0/0.0,
        1/jivertex/-4.0/8.0,
        2/jivertex/4.0/8.0,
        4/divertex/-8.0/10.0,
        5/jivertex/-3.0/11.0,
        6/jivertex/3.0/11.0,
        7/mivertex/-7.0/13.0,
        8/mivertex/7.0/13.0,
        9/mivertex/-6.0/16.0,
        10/mivertex/0.0/16.0,
        11/mivertex/6.0/16.0,
        12/vertex/0.0/24.0
      } \node[\nodetype] (\nodename) at (\xpos, \ypos) {};
    \end{scope}
    \begin{scope} 
      \path (6) edge[conn] (10);
      \path (1) edge[conn] (5);
      \path (7) edge[conn] (9);
      \path (0) edge[conn] (2);
      \path (4) edge[conn] (7);
      \path (5) edge[conn] (9);
      \path (6) edge[conn] (11);
      \path (0) edge[conn] (1);
      \path (5) edge[conn] (10);
      \path (8) edge[conn] (11);
      \path (9) edge[conn] (12);
      \path (0) edge[conn] (4);
      \path (1) edge[conn] (7);
      \path (2) edge[conn] (6);
      \path (2) edge[conn] (8);
      \path (10) edge[conn] (12);
      \path (11) edge[conn] (12);
    \end{scope}
    \begin{scope} 
      \foreach \nodename/\labelpos/\labelopts/\labelcontent in {%
        12/above/n/{W},                                    
        1/below/n/{frog},
        2/below/n/{reed},
        4/below/n/{dog},
        4/above/n/{mammal}, 
        5/below/n/{fish leech, bream},
        5/above//{water animal}, 
        6/below/n/{water weeds},
        6/above//{water plant}, 
        7/above//{land animal}, 
        8/below/n/{corn, bean},
        8/above//{land plant}, 
        9/above/n/{animal}, 
        10/above/n/{LW}, 
        11/above/n/{plant} 
      } \coordinate[label={[\labelopts]\labelpos:{\labelcontent}}](c) at (\nodename);
    \end{scope}
  \end{scope}
\end{tikzpicture}}};
  \end{tikzpicture}
  \caption{A scale context (top), its concept lattice (bottom right)
    for which $\id_G$ is a scale-measure of the context in
    \cref{bjice}. The reflected extents by the scale
    $\sigma^{-1}(\Ext(\Scon))$ of the scale-measure are indicated in
    gray in the contexts concept lattice (bottem left).}
\end{figure}

We present in~\cref{bjicemeasure} the scale-context for some
scale-measure and its concept lattice, derived from our \emph{running
example} context \emph{Living Beings and Water} $\context_{\text{W}}$,
cf.~\cref{bjice}. This scaling is based on the original object set
$G$, however, the attribute set is comprised of nine, partially new,
elements, which may reflect specie taxons. We observe in this example
that the concept lattice of the scale-measure context reflects twelve out
of the nineteen concepts from $\mathfrak{B}(\context_{\text{W}})$.

In our work~\cite{navimeasure} we derived a \emph{scale-hierarchy} on
the set of scale-measures, i.e., $\Sh(\context)\coloneqq \{(\sigma,
\Scon)\mid \sigma$ is a $\Scon-$measure of $\context \}$, from a
natural order of scales introduced by Ganter and
Wille~\cite[Definition 92]{fca-book}). We say for two scale-measures
$(\sigma,\Scon),(\psi,\Tcon)$ that $(\sigma,\Scon)$ is finer then
$(\psi,\Tcon)$, iff $\psi^{-1}(\Ext(\Tcon))\subseteq \sigma^{-1}(\Ext(\Scon))$,
from which also follows a natural equivalence relation $\sim$.

\begin{definition}[Scale-Hierarchy (cf. Definition 7, \cite{navimeasure})]
  For a formal context $\context$ we call $\SH(\context) =
  (\nicefrac{\Sh(\context)}{\sim},\leq)$ the \emph{scale-hierarchy}
  $\context$.
\end{definition}

Also in~\cite{navimeasure}, we have shown that the scale-hierarchy of
a context $\context$ is lattice ordered and isomorphic to the set of
all sub-closure systems of $\Ext(\context)$, i.e., $\{Q\subseteq
\Ext(\context) \mid Q \text{ is a Closure System on }G\}$ that is
ordered by set inclusion $\subseteq$. To show this, we defined a
\emph{canonical representation} of scale-measures, using the so called \emph{canonical scale}
$\context_{\mathcal{A}}\coloneqq (G,\mathcal{A},\in)$ for
$\mathcal{A}\subseteq \Ext(\context)$
with $\Ext(\context_{\mathcal{A}})=\mathcal{A}$.

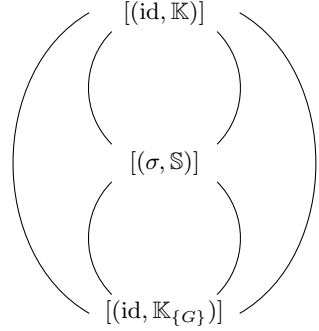
\begin{figure}[t]
  \centering
  \begin{tikzpicture}
    \draw (1,-4) to[out=30, in=-30] (1,0);
    \draw (-1,-4) to[out=150, in=-150] (-1,0);
    \draw (0.7,-3.75) to[out=45, in=-45] (0.7,-2.25);
    \draw (-0.7,-3.75) to[out=135, in=-135] (-0.7,-2.25);
    \draw (0.7,-1.75) to[out=45, in=-45] (0.7,-0.25);
    \draw (-0.7,-1.75) to[out=135, in=-135] (-0.7,-0.25);
    \node[draw opacity = 0,draw=white, text=black] at (0,0)
    {$[(\id, \context)]$};
    \node[draw opacity = 0,draw=white, text=black] at (0,-4)
    {$[(\id,\context_{\{G\}})]$};
    \node[draw opacity = 0,draw=white, text=black] at (0,-2)
    {$[(\sigma, \Scon)]$};
  \end{tikzpicture}
  \caption{Scale-hierarchy of $\context$ with indicated scale-measures.}
  \label{fig:SmAsCl}
\end{figure}

\begin{proposition}[Canonical Representation (cf. Proposition 10, \cite{navimeasure})]\label{prop:eqi-scale}
  Let $\context = (G,M,I)$ be a formal context with scale-measure $(\Scon,\sigma)\in
  \Sh(\context)$, then $(\sigma,\Scon)\sim (\id, \context_{\sigma^{-1}(\Ext(\Scon))})$.
\end{proposition}

We argued in~\cite{navimeasure} that the canonical representation eludes
human explanation to some degree. To remedied this issue by means of
logical scaling \cite{logiscale} which led to to scales with logical
attributes $M_\Scon\subseteq\mathcal{L}(M,\{\wedge,\vee,\neg\})$
(\cite[Problem 1]{navimeasure}).

\begin{proposition}[Conjunctive Normalform
  (cf. Proposition 23, \cite{navimeasure})] \label{lem:appconst}
  Let $\context$ be a context, $(\sigma,\Scon)\in \Sh(\context)$. Then 
  the scale-measure  $(\psi,\Tcon)\in \Sh(\context)$ given by
  \[\psi = \id_G\quad \text{ and }\quad \Tcon = \app\limits_{A\in\sigma^{-1}(\Ext(\Scon))} (G,\{\phi =
    \wedge\ A^{I}\},I_{\phi}) \] is equivalent to $(\sigma,\Scon)$ and
  is called \emph{conjunctive normalform of} $(\sigma,\Scon)$.
  \end{proposition}

\section{Ideals in the Lattice of Closure Systems}\label{ideals}
The goal for the rest of this work is to identify \emph{outstanding}
and particurily interesting data scalings. This quest leads to the
natural question for a structural understanding of the scale-hierarchy
and its elements.  In order to do this we rely on the
isomorphism~\cite[Proposition 11]{navimeasure} between a context's
scale-hierarchy $\SH(\context)$ and the lattice of all sub-closure
systems of the extent set, as explained in the last section. The later
forms an order ideal in the lattice of all closure systems
$\mathfrak{F}_G$ on a set $G$, to which we refer by
$\downarrow_{\mathfrak{F}_G}\Ext(\context)$. This ideal is well
studied~\cite{caspard03} and we may often omit the index
$\mathfrak{F}_G$ to improve the readibility.

Equipped with this structure we have to recall a few notions and
definitions for a complete lattices $(L,\leq)$. In the following, we
denote by $\prec$ the \emph{cover relation} of $\leq$. Furthermore, we say $L$ is
\begin{inparaenum}[1)]
\item \emph{lower
    semi-modular} if and only if $\forall x,y\in L:{x\prec
  x\vee y\implies x\wedge y\prec y}$,
\item\emph{join-semidistributive} iff $\forall x,y,z\in L: {x\vee y = x\vee z} \implies
  x\vee y = x\vee (y\wedge z)$,
\item \emph{meet-distributive} (\emph{lower locally
    distributive}, cf~\cite{caspard03}) iff $L$ is join-semidistributive and lower
  semi-modular,
\item \emph{join-pseudocomplemented} iff
  $x\in L$ the set $\{y\in L \mid y\vee x =\top \}$ has a least,
\item \emph{ranked} iff there is a function $\rho:L\mapsto \mathbb{N}$
  with $x\prec y \implies \rho(x)+1=\rho(y)$,
\item \emph{atomistic} iff
  every $x\in L$ can be written as the
  join of atoms in $L$.
\end{inparaenum}
In addition to the just introduced lattice properties, there are
properties for elements in $L$ that we consider. An element $x\in L$ is 
\begin{inparaenum}[1)]
\item \emph{neutral}
iff every triple $\{x,y,z\}\subseteq L$ generates a distributive
sublattice of $L$,
\item \emph{distributive} iff the equalities
$x\vee (y\wedge z) = (x\vee y)\wedge (x\vee z)$ and
$x\wedge (y\vee z) = (x\wedge y)\vee (x\wedge z)$ for every $y,z\in L$
hold,
\item \emph{meet irreducible} iff $x\neq \top$ and $\bigwedge_{y\in
    Y}y$ for $Y\subseteq L$ implies $x\in Y$,
\item \emph{join irreducible} iff $x\neq \bot$ and $\bigvee_{y\in
    Y}y$ for $Y\subseteq L$ implies $x\in Y$,
\end{inparaenum} Throughout the rest of this work, we denote by
$\mathcal{M}(L)$ the set of all meet-irreducible elements of $L$.

We can derive from literature~\cite[Proposition 19]{caspard03}
the following statement.
\begin{corollary}\label{cor:prec}
  For $\context=(G,M,I)$,  $\downarrow\Ext(\context)\subseteq\mathfrak{F}_G$ and
  $\mathcal{R},\mathcal{R}'\in {\downarrow\Ext(\context)}$ we find the
  equivalence:  $\mathcal{R}'\prec \mathcal{R}\Longleftrightarrow
  \mathcal{R}'\cup \{A\} = \mathcal{R}$ with $A$ is meet-irreducible
  in $\mathcal{R}$
\end{corollary}

Of special interest in lattices are the meet- and join-irreducibles,
since every element of a lattice can be represented as a join or meet
of these elements.

\begin{proposition}\label{prop:join}
  For $\context$, $\downarrow\Ext(\context)\subseteq\mathfrak{F}_G$
  and $\mathcal{R}\in \downarrow \Ext(\context)$ we find the
  equivalence: $\mathcal{R}$ join-irreducible in
  $\downarrow\Ext(\context)\Longleftrightarrow \exists A\in
  \Ext(\context)\setminus\{G\}\colon \mathcal{R}=\{G,A\}$
\end{proposition}
\begin{proof}
  \begin{inparaenum}
  \item[$\Leftarrow$:] For $A\in \Ext(\context)\setminus\{G\}$ is
    $\{A,G\}$ a closure system on $G$ and thereby in $\downarrow
    \Ext(\context)$. Further, the set $\{A,G\}$ is of cardinality two
    and thereby an atom of $\downarrow \Ext(\context)$ and thus
    join-irreducible.
  \item[$\Rightarrow$:] By contradiction assume that ${\not \exists
    A\in\Ext(\context)\setminus\{G\}:\mathcal{R}=\{G,A\}}$, then for
    every $D\in \mathcal{R}\setminus \{G\}$ is $\{D,G\}$ an atom of
    $\downarrow\Ext(\context)$, hence, $\mathcal{R}=\bigvee_{D\in
      \mathcal{R}\setminus \{G\}}\{D,G\}$, i.e., not join-irreducible.
  \end{inparaenum}
\end{proof}

Next, we investigate the meet-irreducibles of
$\downarrow\Ext(\context)$ using a similar approach as done for
$\mathfrak{F}_G$~\cite{caspard03} based on propositional
logic. We recall, that an (object) implication for some context
$\context$ is a pair $(A,B)\in \mathcal{P}(G)\times \mathcal{P}(G)$,
shortly denoted by $A\to B$. We say $A\to B$ is valid in $\context$
iff $A'\subseteq B'$. The set $\mathcal{F}_{A,B}\coloneqq \{D\subseteq
G :A\not\subseteq B \vee B\subseteq D \}$ contains all \emph{models}
of $A\to B$.  Additionally,
$\left.\mathcal{F}_{A,B}\right|_{\Ext(\context)}\coloneqq \mathcal{F}_{A,B}
\cap \Ext(\context)$ is the set of all extents $D\in \Ext(\context)$
that are models of $A\to B$.
The set $\mathcal{F}_{A,B}$ is a closure system~\cite{caspard03}
and therefor $\left.\mathcal{F}_{A,B}\right|_{\Ext(\context)}$,
too. Furthermore, we can deduce that
$\left.\mathcal{F}_{A,B}\right|_{\Ext(\context)}\in \downarrow\Ext(\context)$.

\begin{lemma}\label{prop:meetimpcl}
  For context $\context$, ${\downarrow\Ext(\context)}\subseteq
  \mathfrak{F}_G$, $\mathcal{R}\in {\downarrow\Ext(\context)}$ with
  closure operator $\phi_{\mathcal{R}}$ we find $\mathcal{R} = \bigcap
  \{\left.\mathcal{F}_{A,B}\right|_{\Ext(\context)}\mid A,B\subseteq G \wedge
  B\subseteq \phi_{\mathcal{R}}(A)\}$.
\end{lemma}
\begin{proof}
  We know that $\mathcal{R} = \bigcap \{\mathcal{F}_{A,B}\mid A,B\subseteq G\wedge B\subseteq
  \phi_{\mathcal{R}}(A)\}$~\cite[Proposition 22]{caspard03}. Since
  $\mathcal{R}\subseteq \Ext(\context)$ it holds that $\mathcal{R} =
  \bigcap \{\mathcal{F}_{A,B}\mid A,B\subseteq G\wedge B\subseteq \phi_{\mathcal{R}}(A)\} \cap \Ext(\context)$ and
  thus equal to $\bigcap \{\left.\mathcal{F}_{A,B}\right|_{\Ext(\context)}\mid A,B\subseteq G\wedge B\subseteq
  \phi_{\mathcal{R}}(A)\}$.
\end{proof}

Note that for any $\mathcal{R}\in {\downarrow\Ext(\context)}$ the set
$\{\left.\mathcal{F}_{A,B}\right|_{\Ext(\context)}\mid A,B{\subseteq} G\,\wedge\,
B{\subseteq} \phi_{\mathcal{R}}(A)\}$ contains only closure systems in
$\downarrow \Ext(\context)$ and thus possibly meet-irreducible
elements of $\downarrow \Ext(\context)$.

\begin{proposition}\label{prop:meet}
  For context $\context$, $\downarrow\Ext(\context)\subseteq
  \mathfrak{F}_G$ and $\mathcal{R}\in {\downarrow\Ext(\context)}$, we
  find tfae:
  \begin{inparaenum}
  \item $\mathcal{R}$ is meet-irreducible in
    $\downarrow\Ext(\context)$
  \item $\exists A\in \Ext(\context),i\in G$ with
    $A\prec_{\Ext(\context)}
    (A\cup \{i\})''$ such that $\mathcal{R} =
    \mathcal{F}_{A,\{i\}}|_{\Ext(\context)}$
  \end{inparaenum}
\end{proposition} 
\begin{proof}
  $[1.\Rightarrow 2.]$ Due to \cref{prop:meetimpcl} we can represent
  $\mathcal{R}\in{\downarrow\Ext(\context)}$ by the equation
  $\mathcal{R} = \bigcap \{\mathcal{F}_{A,B}|_{\Ext(\context)}\mid
  A,B\subseteq G \wedge B\subseteq \phi_{\mathcal{R}}(A)\}$. Moreover,
  since $\mathcal{R}$ is meet-irreducible in $\downarrow
  \Ext(\context)$, we can infer that $\mathcal{R} \in
  \{\mathcal{F}_{A,B}|_{\Ext(\context)}\mid A,B\subseteq G \wedge
  B\subseteq \phi_{\mathcal{R}}(A)\}$. In particular there exist
  $A,B\subseteq G$ with $B\subseteq \phi_{\mathcal{R}}(A)$ such that
  $\mathcal{R}=\mathcal{F}_{A,B}|_{\Ext(\context)}$, and thus
  $\mathcal{R}=\mathcal{F}_{A'',B}|_{\Ext(\context)}$. Hence, we
  identify $A''$ by $A$ for the rest of this proof. Using the fact
  that $\mathcal{F}_{A,\{i\}}\cap
  \mathcal{F}_{A,\{j\}}=\mathcal{F}_{A,\{i,j\}}$ we can infer that
  $\mathcal{F}_{A,\{i\}}|_{\Ext(\context)}\cap
  \mathcal{F}_{A,\{i\}}|_{\Ext(\context)}=\mathcal{F}_{A,\{i,j\}}|_{\Ext(\context)}$. Therefore,
  there must exist $A,\{i\}\subseteq G$ with
  $\mathcal{R}=\mathcal{F}_{A,\{i\}}|_{\Ext(\context)}$ ($\ast$).
  
   In the case that $A = (A\cup \{i\})''$ the set
   $\mathcal{F}_{A,\{i\}}|_{\Ext(\context)}=\Ext(\context)$ and
   $\mathcal{R}$ is thereby not meet-irreducible.  Assume that
   $A\not\prec_{\Ext(\context)} (A\cup \{i\})''$, then there is a
   $D\in \Ext(\context)$ with $A\prec_{\Ext(\context)} D \subseteq
   (A\cup \{i\})''$ and $i\not\in D$. Hence $A,D \not\models A\to
   \{i\}$ (see $\ast$) and thus $A,D \not\in \mathcal{R}$. Using this,
   we construct two sets $\mathcal{R}\cup \{A\}$ and $\mathcal{R}\cup
   \{D\}$. The set $\mathcal{R}\cup \{D\}$ is closed by intersection,
   since an intersection of $D$ with an element in $\mathcal{R}$ is a
   model of $A\to i$, thus $\mathcal{R}\cup\{D\} \in\downarrow
   \Ext(\context)$. The same holds for $\mathcal{R}\cup \{A\}$
   resprectively. The intersection of $\mathcal{R}\cup \{A\}$ and
   $\mathcal{R}\cup \{D\}$ is equal to $\mathcal{R}$ which is thereby
   not meet-irreducible, a contradiction.

  $[1.\Leftarrow 2.]$ Consider a closure system $\hat{\mathcal{F}}\in
   \downarrow\Ext(\context)$ with $\hat{\mathcal{F}}$ covers
   $\mathcal{R}$ in ${\downarrow\Ext(\context)}$.  By \cref{cor:prec},
   we can represent $\hat{\mathcal{F}}=\mathcal{R} \cup \{D\}$
   ($\ast$) with $D\not\in \mathcal{R}$ and $D$ is meet-irreducible in
   $\hat{\mathcal{F}}$ (and therefore $D\in \Ext(\context)$). Due to
   $\mathcal{R}\subseteq \hat{\mathcal{F}}$ the set $(A\cup \{i\})''$
   is an element of $\hat{\mathcal{F}}$ and thereby the intersection
   $(A\cup \{i\})''\cap D \in \hat{\mathcal{F}}$. Since $D\not\in
   \mathcal{R}$, we can deduce that $D\not\models A\to i$ and therefor
   $A\subseteq D$ and $i\not\in D$.  From $A\prec_{\Ext(\context)}
   (A\cup \{i\})''$ we know that $(A\cup \{i\})''\cap D = A$. Finally,
   $D\in \hat{\mathcal{F}} \implies A\in \hat{\mathcal{F}}$, and using
   ($\ast$), we can infer that $D=A$.  Hence, $\mathcal{R} \cup \{A\}$
   is the sole upper neighbour of $\mathcal{R}$ in
   $\downarrow\Ext(\context)$ and thereby $\mathcal{R}$ is
   meet-irreducible.
\end{proof}

\cref{prop:join,prop:meet} provide a characterization of irreducible
elements in ${\downarrow}\Ext(\context)$ and thereby in the scale-hierarchy
of $\context$. Those may be of particular interest, since any element
of ${\downarrow}\Ext(\context)$ is representable by irreducible elements.

\begin{proposition}
  For context $\context$, $A,B\in \Ext(\context)$ with
  $A\prec_{\Ext(\context)} B$, then if $A$ is meet-irreducible in
  $\Ext(\context)$, follows
  $\left.\mathcal{F}_{A,B}\right|_{\Ext(\context)}$ is a maximum
  meet-irreducible element in
  $\downarrow\Ext(\context)\subseteq\mathfrak{F}_G$.
\end{proposition}
\begin{proof}
  For $A\prec_{\Ext(\context)} B$, $A$ is the only extent that that is
  not a model of implication $A\to B$, since every other superset of $A$ in
  $\Ext(\context)$ is also a superset of $B$. Hence
  $\left.\mathcal{F}_{A,B}\right|_{\Ext(\context)}$ is equal to
  $\Ext(\context)\setminus \{A\}$. The only superset in
  $\downarrow\Ext(\context)$ is $\Ext(\context)$, which is not
  meet-irreducible.
\end{proof}

Equipped with this characterization we look into counting the irreducibles.

\begin{proposition}\label{prop:meet-num}
  For context $\context$, the number of meet-irreducible elements in
  the lattice ${\downarrow\Ext(\context)}\subseteq\mathfrak{F}_G$ is
  equal to $\mid \prec_{\downarrow\Ext(\context)}\mid$.
\end{proposition}
\begin{proof}
  According to \cref{prop:meet}, an element
  $\mathcal{R}\in \downarrow\Ext(\context)$ is meet-irreducible iff it
  can be represented as $\mathcal{F}_{A,\{i\}}|_{\Ext(\context)}$
  for some $A\in \Ext(\context)$ with
  $A\prec_{\Ext(\context)} (A\cup \{i\})''$. Hence the number of
  meet-irreducible elements is bound by the number of covering pairs
  $A\prec_{\Ext(\context)} B$ in $\Ext(\context)$.  It remains to be
  shown that for $\mathcal{R}$ there is only one pair
  $(A,B)\in \prec_{\Ext(\context)}$ with $B = (A\cup \{i\})''$ for some $i\in B\setminus A$ such that
  $\mathcal{R} = \mathcal{F}_{A,\{i\}}\mid_{\Ext(\context)}$.  Assume
  there are $(A,B),(C,D)\in \prec_{\Ext(\context)}$ with
  $(A,B)\neq (C,D)$ and
  $\mathcal{F}_{A,B}\mid_{\Ext(\context)}=
  \mathcal{F}_{C,D}\mid_{\Ext(\context)}$. First, consider the case
  $A\neq C$. Without loss of
  generality let $A\not\subseteq C$, then we have $C\models A\to B$, but
  $C\not\models C\to D$. Therefore
  $C\in \mathcal{F}_{A,B}\mid_{\Ext(\context)}$ but
  $C\not\in \mathcal{F}_{C,D}\mid_{\Ext(\context)}$. In the second case, $A=C$, we have
  $B\neq D$ and thus $B\not\models C\to D$, but
  $B\models A\to B$. This implies that
  $B\in \mathcal{F}_{A,B}\mid_{\Ext(\context)}$ but
  $B\not\in \mathcal{F}_{C,D}\mid_{\Ext(\context)}$. Thus,
  $\mathcal{F}_{A,B}\mid_{\Ext(\context)}\neq
  \mathcal{F}_{C,D}\mid_{\Ext(\context)}$.
\end{proof}

Next, we turn ourselfs to other lattice properties of
$\downarrow\Ext(\context)$ and its elements.

\begin{lemma}[Join Complement]\label{lem:comp}
  For $\context$, $\downarrow\Ext(\context)\subseteq \mathfrak{F}_G$
  and $\mathcal{R}\in {\downarrow\Ext(\context)}$, the set
  $\hat{\mathcal{R}} = \bigvee_{A\in \mathcal{M}(\Ext(\context))\setminus
    \mathcal{M}(\mathcal{R})} \{A, G\}$ is the inclusion minimum
  closure-system for which
  $\mathcal{R}\vee\hat{\mathcal{R}}=\Ext(\context)$.
\end{lemma}
\begin{proof}
  A set $\mathcal{A}\subseteq\Ext(\context)$ is a generator of
  $\Ext(\context)$ iff all meet-irreducible elements of
  $\Ext(\context)$ are in $\mathcal{A}$. Hence, for every
  $\mathcal{D}\in \downarrow\Ext(\context)$ with $\mathcal{R}\vee
  \mathcal{D} = \Ext(\context)$, we have $D$ is a superset of
  $\mathcal{M}(\Ext(\context))\setminus \mathcal{M}(\mathcal{R})$ and
  thus of $\hat{R}$, since $\hat{R}$ it is the closure of
  $\mathcal{M}(\Ext(\context))\setminus \mathcal{M}(\mathcal{R})$ in
  $\downarrow\Ext(\context)$.
\end{proof}

All the above result in the following statement about $\downarrow\Ext(\context)$:

\begin{proposition}\label{prop:props}
  For context $\context$, the lattice
  $\downarrow\Ext(\context)\subseteq \mathfrak{F}_G$:
\begin{multicols}{2}
  \begin{enumerate}[i)]
  \item is join-semidistributive
  \item is lower semi-modular
  \item is meet-distributive
  \item is join-pseudocomplemented
  \item is ranked
  \item is atomistic
  \end{enumerate}
\end{multicols}
\end{proposition}
\begin{proof}
  \begin{inparaenum}[i)]
  \item According to \cite[Corollary 30]{caspard03} $\mathfrak{F}_{G}$
    is join-semidistributive and therefor $\downarrow \Ext(\context)$
    too, since the meet and join operations of $\mathfrak{F}_{G}$ are closed in $\downarrow
    \Ext(\context)$.
  \item Analogue to i).
  \item Follows from i) and ii) (cf. Definition 15 (5) \cite{caspard03}).
  \item The join-complement of any $\mathcal{R}\in {\downarrow
    \Ext(\context)}$ is given by $\hat{\mathcal{R}}$ according to \cref{lem:comp}.
  \item The lattice $\mathfrak{F}_G$ is ranked by the cardinality
    function (cf.\cite[Corollary 30]{caspard03}). Since
    ${\downarrow \Ext(\context)}$ is an order ideal in $\mathfrak{F}_G$, it
    is ranked by the same function.
  \item Follows directly from the characterization of
    join-irreducibles in \cref{prop:join}.
  \end{inparaenum}
\end{proof}

This result can be employed for the recommendation of scale-measures,
in particular with respect to Libkins decomposition
theorem~\cite[Theorem 1]{Libkin93}. This would allow for a
divide-and-conquer procedure within the scale-hierarchy, based on the
fact: for context $\context$ the lattice
$\downarrow\Ext(\context)\subseteq \mathfrak{F}_G$ is decomposable
into the direct product of two lattices $\downarrow\Ext(\context)\sim
L_1\times L_2$ iff $L_1= (n],L_2= (\overline{n}]$ and $n$ is neutral
    in $\downarrow\Ext(\context)$. Here $\overline{n}$ indicates the
    complement of $n$ with respect to $\downarrow\Ext(\context)$,
    which can be computed using \cref{lem:comp}. That this approach is
    reasonable can be drawn from the fact that
    $\downarrow\Ext(\context)$ fullfils all requirements of Lemma 2
    and Theorem 1 from Libkin's work~\cite{libkin1995direct,Libkin93}
    by considering~\cref{prop:props}.
    
In the rest of this section we investigate distributive and neutral elements in
$\downarrow\Ext(\context)$ more deeply. For this, let $\psi,\phi\in
\Phi(L)$, i.e., the set of all closure operators on lattice $L$. We
say that $\phi \leq_{\Phi} \psi$ iff for all $x\in L:
\phi(x)\leq_{\Phi} \psi(x)$.

\begin{lemma}\label{lem:dual-iso}
  For context $\context$,
  $\downarrow\Ext(\context)\subseteq\mathfrak{F}_G$ and
  $\Phi(\Ext(\context))$, we find that the map
  $i:\downarrow\Ext(\context) \mapsto \Phi(\Ext(\context))$ with
  $i(\mathcal{A})\to \phi_{\mathcal{A}}|_{\Ext(\context)}$ is a
  dual-isomorphism.
\end{lemma}
\begin{proof}
  For $\mathcal{A},\mathcal{D}\in \downarrow\Ext(\context)$ with
  $A\in \mathcal{A},A\not\in \mathcal{D}$ is $i(\mathcal{A})(A)=A$ but
  $i(\mathcal{D})(A)\neq A$. Thus $i(\mathcal{A})\neq i(\mathcal{D})$
  and $i$ injective. For $\phi\in \Phi(\Ext(\context))$ is
  $\phi[\Ext(\context)]\subseteq\Ext(\context)$ a closure system
  with $G\in \phi[\Ext(\context)]$ an therefor
  $\phi[\Ext(\context)]\in \downarrow\Ext(\context)$ with
  $i(\phi[\Ext(\context)]) = \phi$. Hence $i$ is bijective. For
  $\mathcal{A},\mathcal{D}\in \downarrow\Ext(\context)$ with
  $\mathcal{A}\prec_{\downarrow\Ext(\context)}\mathcal{D}$ is
  $\mathcal{A}\cup \{D\}=\mathcal{D}$ for $D$ meet-irreducible in
  $\mathcal{D}$ (\cref{cor:prec}). Thus for all $A\in \Ext(\context)$ is
  $i(\mathcal{A})(A) = i(\mathcal{D})(A)$ except for the pre-images of
  $D$, i.e., $i(\mathcal{D})^{-1}(D)$. For
  $A\in i(\mathcal{D})^{-1}(D)$ is
  $i(\mathcal{D})(A) = D \subseteq i(\mathcal{A})(A)$ and thus
  $i(\mathcal{D})\leq i(\mathcal{A})$, as required.
\end{proof}

\begin{corollary}\label{rem:neutral}
  For a context $\context$,
  $\downarrow\Ext(\context)\subseteq \mathfrak{F}_G$ and
  $\mathcal{R}\in \downarrow\Ext(\context)$ tfae:
  \begin{inparaenum}[i)]
  \item $\mathcal{R}$ is distributive
  \item $\mathcal{R}$ is neutral
  \item For $A,B,C\in\mathcal{R}$ with $C = A \wedge B$ and $A,B$
    incomparable in $\Ext(\context)$, we have $A\in\mathcal{R}\vee
    B\in\mathcal{R}\vee C\in \mathcal{R}$ implies $A,B,C\in\mathcal{R}$.
  \end{inparaenum}

\end{corollary}
\begin{proof}
  Using \cref{lem:dual-iso},
  i)$\Leftrightarrow$ii) due to Thm.\,2 \cite{morgado1963distributive} and 
  i)$\Leftrightarrow$iii) due to  Thm.\,1 \cite{morgado1963distributive}. 
\end{proof}

An additional accompanying property is that the set of neutral
elements of $\downarrow\Ext(\context)$ is a complete
lattice~\cite{morgado1964central}. Thus, the iterative procedure
that results from~ \cref{rem:neutral}, iii) yields a closure operator on
$\downarrow\Ext(\context)$ to compute the neutral elements. To nourish our
understanding of the neutral elements take the following example: in
the lattice $\mathfrak{F}_G$ are only the top and bottom elements
neutral~\cite[Proposition 33 (5)]{caspard03}. In contrast, take the
chain $\mathcal{C}\subseteq \mathfrak{F}_G$ with $G\in \mathcal{C}$,
for which $\downarrow\Ext(\context_{\mathcal{C}})$ is a distributive
lattice, hence, every element is neutral.


\section{Recommending Conceptual Scale-Measures}\label{sec4}

Our theoretical findings unvails several possibilities to recommend
scale-measures. First, there are meet- and join-irreducible elements
of the scale-hierarchy (\cref{prop:meet,prop:join}). These elements
are a minimum representation from which every other scale-measure can
be retrieved. However, the number of meet- and join-irreducible
elements is in the size of the concept lattice $\BV(\context)$
(\cref{prop:join}) and thereby potentially exponential large.  Hence,
it is necessary to narrow down the set of join-irreducible scale-measures, 
for example, by constraining the selection  to irreducible elements in
$\BV(\context)$ or by applying conceptual importance measure. 

Other scale-measures of interest can be depicted based on their
structural placement in the scale-hierarchy, i.e., element-wise
modularity, distributivity, or neutrality. A further advantage of
latter two selection methods is that they allow a decomposition of the
scale-hierarchy using divide-and-conquer strategies. The existence of
such neutral elements, however, cannot be guarantied, as it can be
observed in $\mathfrak{F}_G$.
When a starting scale-measure $(\sigma,\Scon)$ is selected, an obvious
choice is to recommend the join-complemented scale-measure
(\cref{prop:props}), i.e., the minimum scale-measure such that the
join with $(\sigma,\Scon)$ yields $\Ext(\context)$.  The said
join-complemented scale-measure can then be used as additional
information or be the starting point for a thorough search.

In general, whenever multiple scale-measures of interest
$\left\{(\sigma_j,\Scon_j)\right\}_{j\in J}$ are selected, we are able to
combine all those  by the apposition of scale-measures (\cite[Proposition
  19]{navimeasure}) to combine their conceptual views on the data set.

\subsection{Exploration}\label{sec:explore}

\begin{algorithm}[t]
          \SetKwInOut{Input}{Input}
          \SetKwInOut{Output}{Output}
          \SetKwInOut{Return}{return}

          \Input{ Context $\context = (G,M,I)$}
          \Output{$(id_G,\Scon)\in \Sh(\context)$ and optionally $\mathcal{L}_{\Scon}$}
          Init Scale $\Scon = (G,\emptyset,\in)$\color{gray!60}\\
          Init $A= \emptyset$,\color{black}
          $\mathcal{L}_{\Scon}=\text{CanonicalBase}(\context)$ (or
          $\mathcal{L}_{\Scon}=\{\}$ for larger contexts)\color{gray!60}\\
        \While{ $A\neq G$}{
        \While{$A\neq A^{I_{\Scon}I_{\Scon}})$}{\color{black}
        \eIf{Can 
          $A^{I_{\context}} \setminus (A)^{I_{\Scon}I_{\Scon}I_{\context}}$
          for objects having $(A)^{I_{\Scon}I_{\Scon}I_{\context}}$ be neglected?}{
          \color{gray!60}
         $\mathcal{L}_{\Scon}=\mathcal{L}_{\Scon}\cup \{A\to A^{I_{\Scon}I_{\Scon}}\}$ \\
         Exit While \\
         }{ \color{black}
           Enter $B\subseteq A^{I_{\context}} \setminus
           (A)^{I_{\Scon}I_{\Scon}I_{\context}}$ that should be
           considered\\ Add attribute $B^{I_{\context}}$ to
           $\Scon$\color{gray!60} } }

       $A = $\texttt{Next\_Closure}$(A,G,\mathcal{L}_{\Scon})$ }
        \color{black} \Return{$(id_{G},\Scon)$ and optionally
          $\mathcal{L}$}
       \caption{Scale-measure Exploration: A modified Exploration with
         Background Knowledge}
       \label{alg:expl}
\end{algorithm}

For the task of efficiently determining a scale-measure, based on
human preferences, we propose the following approach. Motivated by the
representation of meet-irreducible elements in the scale-hierarchy
through object implications of the context (\cref{prop:meet}), we
employ the dual of the \emph{attribute exploration} algorithm
(\cite{exploration}) by \citeauthor{exploration}. We modified said
algorithm toward exploring scale-measures and present its pseudo-code
in~\cref{alg:expl}. In this depiction we highlighted our modifications
with respect to the original exploration algorithm (Algorithm 19,
\cite{concexpl}) with darker print. This algorithm semi-automatically
computes a scale context $\Scon$ and its canonical base. In each
iteration of the inner loop of our exploring algorithm the query that
is stated to the \emph{scaling expert} is if an object implication
$A\implies B$ is true in the closure system of preferences. If the
implication holds, it is added to the implicational base of $\Scon$
and the algorithm continues with the next implication query. Otherwise
a counter example in the form of a closed set $C\in \Ext(\context)$
with $A\subseteq C$ but $B\not\subseteq C$. This closed set is then
added as attribute to the scale context $\Scon$ with the incidences
given by $\in$. If $C\not\in\Ext(\context)$ the scale $\Scon$ would
contradict the scale-measure property (Proposition 20,
\cite{navimeasure}).

The object implicational theory $L_\Scon$ is initialized to the object
canonical base of $\context$, which is an instance of according to
attribute exploration with background knowledge
\cite{exploration}. This initialization can be neglected for larger
contexts, however it may reduce the number of queries. The algorithm
terminates when the implication premise of the query is equal to $G$.
The returned scale-measure is in canonical form, i.e., the canonical
representation $(id_G,(G,\Ext(\Scon),\in))$ (\cref{prop:eqi-scale}).
The motivation behind attribute exploration queries is to determine if
an implication holds in the unknown representational context of the
learning domain. In contrast, the exploration of scale-measures
determines if a given $\Ext(\context)$ can be coarsened by
implications $A\implies B$, resulating in a smaller and thus more
human comprehensible concept lattice $\BV(\Scon)$, adjusted to the
preferences (or view) of the scaling expert.

Querying object implications may be less intuitive compared to
attribute implications, hence, we suggest to rather not test for
$A\implies A^{I_{\Scon}I_{\Scon}}$ for $A\subseteq G$ but to test if
the difference of the intents $A^{I_\mathbb{K}}$ and
$(A^{I_{\Scon}I_{\Scon}})'$ in $\context$, is of relevance to the
scaling expert. Finally, as a post-processing, one may apply the
\emph{conjunctive normalform}~\cite[Proposition 23]{navimeasure} of
scale-measures to further increase the human-comprehension. Yet,
deriving other human-comprehensible representations of scale-measures
is deemed future work.

\subsection{(Semi-)Automatic Large Data Set
  Scaling}
\label{sec:automatic}
\label{sec:5}
\begin{figure}
  \centering 
\tiny
\begin{tikzpicture}
\node at (0,0){
\rotatebox{90}{\begin{tabular}{|l|l|lcr|c|l|}
                 \hline (Object) Premise (1) & (Object) Conclusion (2) & \multicolumn{3}{c|}{Attribute Question}  & CE? & Edit\\  \hline 
&&\ $(2)'$&&  $(1)'\setminus (2)'$&&\\
\hline \hline
$\{\}$&$\{\text{\color{blue} D, FL, Co, Br, WW, Be, F, R}\}$
&$\{\text{\color{red}W}\}$: & add more detail using&
$\{\text{\color{red}L, BF, Ch, LL, LW, M, MC, DC}\}$?&$\{\text{\color{red}M}\}$
&Add \hfill$\{\text{\color{blue} D, FL, Br, F}\}$ to $M_\Scon$\\ 
\hline
&$\{\text{\color{blue}D, FL, Br, F}\}$ & $\{\text{\color{red}W, M}\}$
: &add more detail using& $\{\text{\color{red}L, BF, Ch, LL, LW, MC, DC}\}$?
&$\{\text{\color{red}LL}\}$&Add \hfill$\{\text{\color{blue}D, F}\}$ to $M_\Scon$\\
\hline
&$\{\text{\color{blue}D, F}\}$&$\{\text{\color{red}L, W, LL, M}\}$:& add more detail using &$\{\text{\color{red}BF, Ch, LW, MC, DC}\}$?&$\{\text{\color{red}BF}\}$&Add \hfill$\{\text{\color{blue}D}\}$ to $M_\Scon$\\ 
\hline
&$\{\text{\color{blue}D}\}$&$\{\text{\color{red}L, BF, W, LL, M}\}$:& add more detail using& $\{\text{\color{red}Ch, LW, MC, DC}\}$?&$\{\text{\color{red}LW}\}$&Add \hfill$\{\}$ to $M_\Scon$\\ 
\hline
&$ \{\}$&&&&&\\ 
\hline
$\{\text{\color{blue}R}\}$&$\{\text{\color{blue}D, FL, C, Br, WW, Be, F, R}\}$&$\{\text{\color{red}W}\}$:& add more detail using &$\{\text{\color{red}Ch, LL, LW, MC}\}$?&$\{\text{\color{red}Ch}\}$&Add \hfill $\{\text{\color{blue} C, WW, Be, R}\}$ to $M_\Scon$\\ 
\hline
&$\{\text{\color{blue} C, WW, Be, R}\}$&$\{\text{\color{red}W, Ch}\}$:& add more detail using &$\{\text{\color{red}LL, LW, MC}\}$?&$\{\text{\color{red}LL}\}$&Add \hfill $\{\text{\color{blue}C, Be, R}\}$ to $M_\Scon$\\ 
\hline
&$\{\text{\color{blue} C, Be, R}\}$&$\{\text{\color{red}W, Ch, LL}\}$:& add more detail using &$\{\text{\color{red}LW, MC}\}$?&$\{\text{\color{red}LW}\}$&Add \hfill $\{\text{\color{blue}R}\}$ to $M_\Scon$\\ 
\hline
&$\{\text{\color{blue}R}\}$&&&&&\\ 
\hline
$\{\text{\color{blue}F}\}$ &$\{\text{\color{blue}F, D}\}$ & $\{\text{\color{red}L, W, LL, M}\}$:& add more detail using&$\{\text{\color{red}LW}\}$? &$\{\text{\color{red}LW}\}$ &Add \hfill $\{\text{\color{blue}F}\}$ to $M_\Scon$\\ 
\hline
&$\{\text{\color{blue}F}\}$&&&&&\\ 
\hline
$\{\text{\color{blue}F, R}\}$ &$\{\text{\color{blue}D, FL, Co, Br, WW, Be, F, R}\}$ & $\{\text{\color{red}W}\}$:& add more detail using&$\{\text{\color{red}LW, LL}\}$? &$\{\text{\color{red}LW}\}$ &Add \hfill $\{\text{\color{blue}FL, Br, WW, F, R}\}$ to $M_\Scon$\\ 
\hline
&$\{\text{\color{blue}FL, Br, WW, F, R}\}$ & $\{\text{\color{red}W}\}$:& add more detail using&$\{\text{\color{red}LW, LL}\}$? &no&Add\hfill $(1){\implies} (2)$ to $\mathcal{L}_{\Scon}$\\ 
\hline
$\{\text{\color{blue}F, Br}\}$ &$\{\text{\color{blue}FL, Br, F}\}$ & $\{\text{\color{red}W, LW, M}\}$:& add more detail using&$\{\text{\color{red}L}\}$? &no&Add\hfill $(1){\implies} (2)$ to $\mathcal{L}_{\Scon}$\\ 
\hline
$\{\text{\color{blue}FL, Br, F}\}$&$\{\text{\color{blue}FL, Br, F}\}$&&&&&\\ 
\hline
$\{\text{\color{blue}WW, R}\}$&$\{\text{\color{blue}WW, R}\}$&&&&&\\ 
\hline
$\{\text{\color{blue}FL, Br, WW, F, R}\}$&$\{\text{\color{blue}FL, Br, WW, F, R}\}$&&&&&\\ 
\hline
$\{\text{\color{blue}D}\}$&$\{\text{\color{blue}D}\}$&&&&&\\ 
\hline
$\{\text{\color{blue}D, F}\}$&$\{\text{\color{blue}D, F}\}$&&&&&\\ 
\hline
$\{\text{\color{blue}D, FL, Br, F}\}$&$\{\text{\color{blue}D, FL, Br, F}\}$&&&&&\\ 
\hline
$\{\text{\color{blue}Co, R}\}$ &$\{\text{\color{blue}Co, Be, R}\}$ & $\{\text{\color{red}W, Ch, LL}\}$:& add more detail using&$\{\text{\color{red}Mc}\}$? &no&Add\hfill $(1){\implies} (2)$ to $\mathcal{L}_{\Scon}$\\ 
\hline
$\{\text{\color{blue}Be}\}$ &$\{\text{\color{blue}Co, Be, R}\}$ & $\{\text{\color{red}W, Ch, LL}\}$:& add more detail using&$\{\text{\color{red}Dc}\}$? &no&Add\hfill $(1){\implies} (2)$ to $\mathcal{L}_{\Scon}$\\ 
\hline
$\{\text{\color{blue}Co, Be, R}\}$&$\{\text{\color{blue}Co, Be, R}\}$&&&&&\\ 
\hline
$\{\text{\color{blue}Co, WW, Be, R}\}$&$\{\text{\color{blue}Co, WW, Be, R}\}$&&&&&DONE\\ 
\hline
  \end{tabular}}};
\node at (6,-3){
\rotatebox{90}{
  \scalebox{1}{\scriptsize
      \begin{cxt}
        \cxtName{}
        \att{{$\{\text{\color{blue}D, FL, Br, F}\}$}}
        \att{{$\{\}$}}
        \att{{$\{\text{\color{blue}Co, WW, Be, R}\}$}}
        \att{{$\{\text{\color{blue}Co, Be, R}\}$}}
        \att{{$\{\text{\color{blue}FL, Br, WW, F, R}\}$}}
        \att{{$\{\text{\color{blue}D, F}\}$}}
        \att{{$\{\text{\color{blue}D}\}$}}
        \att{{$\{\text{\color{blue}F}\}$}}
        \att{{$\{\text{\color{blue}R}\}$}}
        \obj{x....xx..}{dog\hfill ({\color{blue}D})}
        \obj{x...x....}{fish leech\hfill ({\color{blue}FL})}
        \obj{..xx.....}{corn\hfill ({\color{blue}Co})}
        \obj{x...x....}{bream\hfill ({\color{blue}Br})}
        \obj{..x.x....}{water weeds\hfill ({\color{blue}WW})}
        \obj{..xx.....}{bean\hfill ({\color{blue}Be}) }
        \obj{x...xx.x.}{frog\hfill ({\color{blue}F}) }
        \obj{..xxx...x}{reed\hfill ({\color{blue}R}) }
      \end{cxt}}}};
\node at (6,7){
  \rotatebox{90}{\colorlet{mivertexcolor}{black!80}
\colorlet{jivertexcolor}{black!80}
\colorlet{vertexcolor}{black!80}
\colorlet{bordercolor}{black!80}
\colorlet{linecolor}{gray}
\tikzset{vertexbase/.style={semithick, shape=circle, inner sep=2pt, outer sep=0pt, draw=bordercolor},%
  vertex/.style={vertexbase, fill=vertexcolor!45},%
  mivertex/.style={vertexbase, fill=mivertexcolor!45},%
  jivertex/.style={vertexbase, fill=jivertexcolor!45},%
  divertex/.style={vertexbase, top color=mivertexcolor!45, bottom color=jivertexcolor!45},%
  conn/.style={-, thick, color=linecolor}%
}
\tikzstyle{n} = [text width=1.3cm,align=center]
\begin{tikzpicture}[scale=0.2,font=\tiny]
  \begin{scope} 
    \begin{scope} 
      \foreach \nodename/\nodetype/\xpos/\ypos in {%
        0/vertex/0.0/0.0,
        1/jivertex/-4.0/8.0,
        2/jivertex/4.0/8.0,
        4/divertex/-8.0/10.0,
        5/jivertex/-3.0/11.0,
        6/jivertex/3.0/11.0,
        7/mivertex/-7.0/13.0,
        8/mivertex/7.0/13.0,
        9/mivertex/-6.0/16.0,
        10/mivertex/0.0/16.0,
        11/mivertex/6.0/16.0,
        12/vertex/0.0/24.0
      } \node[\nodetype] (\nodename) at (\xpos, \ypos) {};
    \end{scope}
    \begin{scope} 
      \path (6) edge[conn] (10);
      \path (1) edge[conn] (5);
      \path (7) edge[conn] (9);
      \path (0) edge[conn] (2);
      \path (4) edge[conn] (7);
      \path (5) edge[conn] (9);
      \path (6) edge[conn] (11);
      \path (0) edge[conn] (1);
      \path (5) edge[conn] (10);
      \path (8) edge[conn] (11);
      \path (9) edge[conn] (12);
      \path (0) edge[conn] (4);
      \path (1) edge[conn] (7);
      \path (2) edge[conn] (6);
      \path (2) edge[conn] (8);
      \path (10) edge[conn] (12);
      \path (11) edge[conn] (12);
    \end{scope}
    \begin{scope} 
      \foreach \nodename/\labelpos/\labelopts/\labelcontent in {%
        0/above/n/{$\{\}$},                                    
        1/below/n/{frog},
        2/below/n/{reed},
        2/above/n/{$\{${\color{blue}R}$\}$},
        4/below/n/{dog},
        4/above/n/{$\{${\color{blue}D}$\}$},
        5/below/n/{fish leech, bream},
        6/below/n/{water weeds},
        8/below/n/{corn, bean},
        8/above//{{$\{$\color{blue}Co, Be, R$\}$}}, 
        9/above/n/{{$\{$\color{blue}D, FL, Br, F$\}$}}, 
        10/above/n/{{$\{$\color{blue}FL, Br, WW, F, R$\}$}}, 
        11/above/n/{{$\{$\color{blue}Co, WW, Be, R$\}$}}, 
        4/above/n/{{$\{$\color{blue}D$\}$}}, 
        7/above//{{$\{$\color{blue}D, F$\}$}} 
      } \coordinate[label={[\labelopts]\labelpos:{\labelcontent}}](c) at (\nodename);
    \end{scope}
  \end{scope}
\end{tikzpicture}}};
\end{tikzpicture}
  \caption{Scale-measure exploration results (left) for the
    \emph{Living Beings and Water}
    context, the resulting context (bottom right) and its concept
    lattice (top right). \\The employed object order is:
    $\text{\color{blue} Be} > \text{\color{blue} Co} > \text{\color{blue} D}
    > \text{\color{blue} WW} > \text{\color{blue} FL} > \text{\color{blue} Br} > \text{\color{blue} F} > \text{\color{blue} R}$}
  \label{fig:expl-table}
\end{figure}
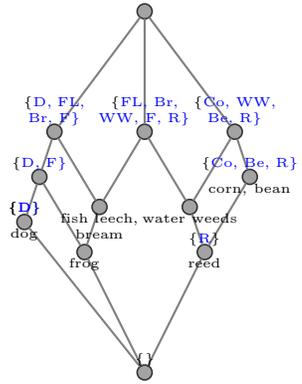

To demonstrate the applicability of the presented exploring algorithm,
we have implemented it in the \texttt{conexp-clj}(\cite{conexp})
software suite for formal concept analysis.For this, we apply the
scale-measure exploration \cref{alg:expl} on our running example
$\context_{W}$, see~\cref{fig:bj1}. In~\cref{fig:expl-table} (left) we
depicted the evaluation steps of algorithm, the first two columns
represent the object implication that is queried, the third column
contains the query translated in terms of attributes. For example, in
row two the implication $\{\}\implies \{\text{\color{blue}D},
\text{\color{blue}FL}, \text{\color{blue}Br}, \text{\color{blue}F}\}$
is true in the so far generated scale $\Scon$ and is queried if it
should hold.  All objects of the implication do have at least the
attribues \emph{can move} and \emph{needs water to live}, as indicated
in the third column (left). In the same column (right) we find
attributes from $(1)^{I_{\context}}\setminus
(2)^{I_{\context}}\subseteq M_{W}$ that can be considered by the
scaling expert to narrow the object implication, i.e., to shrinken the
size of the conclusion. The by us envisioned answer of the scaling
expert is given in column four, the attribute \emph{lives on
land}. Thus, the object counter example is then the
attribute-derivation the union $\{\text{\color{red} M},
\text{\color{red}W}, \text{\color{red}
  LL}\}^{I_W}=\{\text{\color{blue}D},\text{\color{blue}F}\}$. In our
example of the scale-measure exploration the algorithm terminates
after the scaling expert provided in total nine counter examples and four accepts.  The
output is a scale context in canonical representation with
twelve concepts as depicted in~\cref{fig:expl-table} (right).

The just demonstrated application of the scale-measure exploration can
be supported in every steop by conceptual importance
measures~\cite{measures}.  Furthermore, these measures can also be
used to automate the exploration algorithm by randomly selecting the
counterexample from the top-k of the list of outstanding concepts with
respect to one or more of said conceptual measures. We show illustrate
this idea for the spices planer data
set~\cite{herbs,pqcores,navimeasure} and depict the resulting scale-measure
in~\cref{fig:rec}. This data set is comprised of 56 dishes (objects)
and 37 spices (attributes), resulting in the context
$\context_{\text{Spices}}$. The dishes are picked from multiple
categories, such as vegetables, meats, or fish dishes. The incidence
$I_{\context_{\text{Spices}}}$ indicates that a spice $m$ is necessary
to cook dish $g$. The concept lattice of $\context_{\text{Spices}}$
has 421 concepts and is therefore too large for a meaningful human
comprehension. Thus, using our automatic approach for scale-measure
recommendation, we are able to generate a small-scaled view of
readable size.

For this example of automatic scale-measure exploration, we considered
the importance measure \emph{separation index}~\cite{measures, prob}
on the set of objects. We consider the maximum number of concepts that
are human readable to be thirty and therefore we restricted the number
of counter examples to be computed accodingly.  We depicted the
concept lattice of the resulting scale-measure in \cref{fig:rec} using
the conjunctive normalform.  To improve the readability, we only
annotated meet-irreducible attribute concepts in the lattice diagram
and omitted redundant attribute conjunctions, e.g., for
Anis$\wedge$Vanilla$\wedge$Cinnamon$\wedge$Pastry we annotate
$...\wedge$Pastry, since Anis$\wedge$Vanilla$\wedge$Cinnamon is
already given by an upper neighbor. The so given scale-measure concept
lattice seems empirically more human readable and displays extensive
information with respect to the original data set
$\context_{\text{Spices}}$ and the employed importance measure.

Based on this approach, we propose a comprehensive study that
specifically examines the use of the different importance measures in
relation to the data domains used. Such a study would, of course, go
beyond the scope of this paper. Another approach to improve the
automatic scaling process could be the removal of irrelevant
attributes (\cite{DBLP:conf/iccs/HanikaKS19}). Other selection
criteria could regard the distributivity of concepts, since
distributive lattices are known to have easy readable drawings.
Another, line of research with respect to improving the automatic
scaling with our algorithm regards the logical representation of the
scale-measure attributes. In presented work, we use the conjunctive
normalform, but future work may investigate new and additional logical
representations.

\begin{figure}
  \centering
  \scalebox{0.43}{\colorlet{mivertexcolor}{black!80}
\colorlet{jivertexcolor}{black!80}
\colorlet{vertexcolor}{black!80}
\colorlet{bordercolor}{black!80}
\colorlet{linecolor}{gray}
\tikzset{vertexbase/.style={semithick, shape=circle, inner sep=2pt, outer sep=0pt, draw=bordercolor},%
  vertex/.style={vertexbase, fill=vertexcolor!45},%
  mivertex/.style={vertexbase, fill=mivertexcolor!45},%
  jivertex/.style={vertexbase, fill=jivertexcolor!45},%
  divertex/.style={vertexbase, top color=mivertexcolor!45, bottom color=jivertexcolor!45},%
  conn/.style={-, thick, color=linecolor}%
}
\tikzstyle{o} = [text width=2cm,align=center]
\tikzstyle{n} = [text width=2.2cm,align=center]
\tikzstyle{l} = [text width=4cm, label distance=1cm,align=center]
\tikzstyle{m} = [text width=2.5cm,align=center]
\begin{tikzpicture}[scale=0.5,font=\footnotesize]
  \begin{scope} 
    \begin{scope} 
      \foreach \nodename/\nodetype/\xpos/\ypos in {%
        0/vertex/0.0/0.0,
        1/jivertex/-6.0/16.0,
        2/jivertex/0.0/16.0,
        3/jivertex/4.0/16.0,
        4/divertex/23.917664092664076/19.639382239382236,
        5/jivertex/-14.0/20.0,
        6/vertex/2.0/20.0,
        7/jivertex/-20.0/24.0,
        8/divertex/23.856081081081086/24.196525096525093,
        9/vertex/-14.0/26.0,
        10/jivertex/-2.0/26.0,
        11/divertex/19.606853281853276/26.659845559845557,
        12/jivertex/-25.0/27.0,
        13/vertex/-20.0/28.0,
        14/vertex/4.0/28.0,
        15/jivertex/9.0/29.0,
        16/divertex/24.0/30.0,
        17/mivertex/15.911872586872583/31.52490347490347,
        18/vertex/-4.0/32.0,
        19/mivertex/0.14662162162161252/33.00289575289575,
        20/mivertex/-24.609749034749033/33.74189189189188,
        21/mivertex/5.0/35.0,
        22/mivertex/9.0/35.0,
        23/mivertex/-18.0/36.0,
        24/mivertex/-10.0/36.0,
        25/divertex/23.917664092664097/36.51312741312741,
        26/mivertex/-3.0/39.0,
        27/mivertex/-10.0/40.0,
        28/mivertex/2.0/40.0,
        29/vertex/0.0/58.0
      } \node[\nodetype] (\nodename) at (\xpos, \ypos) {};
    \end{scope}
    \begin{scope} 
      \path (5) edge[conn] (13);
      \path (16) edge[conn] (25);
      \path (25) edge[conn] (29);
      \path (3) edge[conn] (17);
      \path (14) edge[conn] (22);
      \path (15) edge[conn] (22);
      \path (13) edge[conn] (27);
      \path (3) edge[conn] (6);
      \path (8) edge[conn] (16);
      \path (10) edge[conn] (21);
      \path (26) edge[conn] (29);
      \path (0) edge[conn] (7);
      \path (18) edge[conn] (26);
      \path (5) edge[conn] (9);
      \path (1) edge[conn] (9);
      \path (19) edge[conn] (26);
      \path (10) edge[conn] (24);
      \path (6) edge[conn] (10);
      \path (10) edge[conn] (19);
      \path (24) edge[conn] (27);
      \path (20) edge[conn] (29);
      \path (9) edge[conn] (26);
      \path (14) edge[conn] (19);
      \path (2) edge[conn] (23);
      \path (0) edge[conn] (4);
      \path (6) edge[conn] (15);
      \path (12) edge[conn] (23);
      \path (0) edge[conn] (3);
      \path (27) edge[conn] (29);
      \path (28) edge[conn] (29);
      \path (18) edge[conn] (28);
      \path (7) edge[conn] (13);
      \path (15) edge[conn] (21);
      \path (4) edge[conn] (8);
      \path (10) edge[conn] (18);
      \path (0) edge[conn] (11);
      \path (13) edge[conn] (20);
      \path (5) edge[conn] (18);
      \path (18) edge[conn] (27);
      \path (6) edge[conn] (14);
      \path (23) edge[conn] (29);
      \path (17) edge[conn] (29);
      \path (22) edge[conn] (29);
      \path (21) edge[conn] (28);
      \path (9) edge[conn] (20);
      \path (7) edge[conn] (24);
      \path (0) edge[conn] (5);
      \path (2) edge[conn] (6);
      \path (12) edge[conn] (20);
      \path (1) edge[conn] (14);
      \path (11) edge[conn] (17);
      \path (0) edge[conn] (1);
      \path (0) edge[conn] (12);
      \path (0) edge[conn] (2);
    \end{scope}
    \begin{scope} 
      \foreach \nodename/\labelpos/\labelopts/\labelcontent in {%
        1/below/n/{Fried-Fish},
        2/below/n/{Vegtables, Stew},
        3/below/n/{Dark-Sauce},
        4/below/n/{Bowle, Cake, Jam, Tea, Christmas-Pastry},
        4/above/n/{...$\wedge$Ginger},
        5/below/n/{Herb-Dip, Hash},
        6/below/n/{Fried-Potato, Pizza, Stove-Potato},
        7/below/n/{Duck, Pork},
        8/below/n/{Dessert},
        8/above/n/{...$\wedge$Cloves},
        9/below/n/{Asian-Rice},
        10/below/n/{Pasta},
        11/below/n/{Red-Cabbage, Sauerbraten, Wild},
        11/above/n/{...$\wedge$Cloves},
        12/below/n/{Veggie-Casserole},
        13/below/n/{Chicken},
        15/below/n/{Beef},
        16/below/n/{Fruit-Salad},
        16/above/n/{...$\wedge$Pastry$\wedge$Sweets},
        17/above/n/{Allspice$\wedge$Lorber$\wedge$ Juniper-Berries $\wedge$Pepper-Black},
        18/below/n/{Goulash},
        19/above/n/{...$\wedge$Thyme},
        20/below/n/{Steamed-Fish, Paella, White-Sauce, Grilled-Fish, Baked-Fish, Veal-Meat},
        20/above/n/{Curry$\wedge$Garlic $\wedge$Pepper-White},
        21/below/n/{Mushrooms},
        21/above/n/{...$\wedge$Thyme},
        22/below/n/{Lamb},
        22/above/n/{Thyme$\wedge$Rosemary$\wedge$ Garlic$\wedge$Cayenne},
        23/below/m/{Carrots, Green-Salad, Spinach, Broccoli, Cauliflower, Lentil-Soup, Cucumber-Salad, Beans, Sauerkraut, Tomato-Salad, Kohlrabi},
        23/above/n/{Vegetables $\wedge$Pepper-White},
        24/above/n/{...$\wedge$Thyme},
        25/below/n/{Rice-Pudding},
        25/above/n/{Anise$\wedge$Vanilla $\wedge$Cinnamon},
        26/above/n/{Paprika-Roses $\wedge$Garlic$\wedge$Cayenne},
        27/above/n/{Paprika-Roses$\wedge$Paprika-Sweet$\wedge$Garlic},
        28/above/n/{Paprika-Sweet$\wedge$Garlic $\wedge$Cayenne},
        29/below/l/{Omlette, Potato-Gratin, Puree, Cheese-Pastry, Potato-Soup, Shellfish, Roulades, Goose, Curry-Rice}
      } \coordinate[label={[\labelopts]\labelpos:{\labelcontent}}](c) at (\nodename);
    \end{scope}
  \end{scope}
\end{tikzpicture}}
  \caption{Automatically generated scale-measure of the spices context
    using the most outstanding concepts by the separation index
    importance measure. The scale has consists of 30 of the original
    421 concepts and is in conjunctive normalform.}
  \label{fig:rec}
\end{figure}
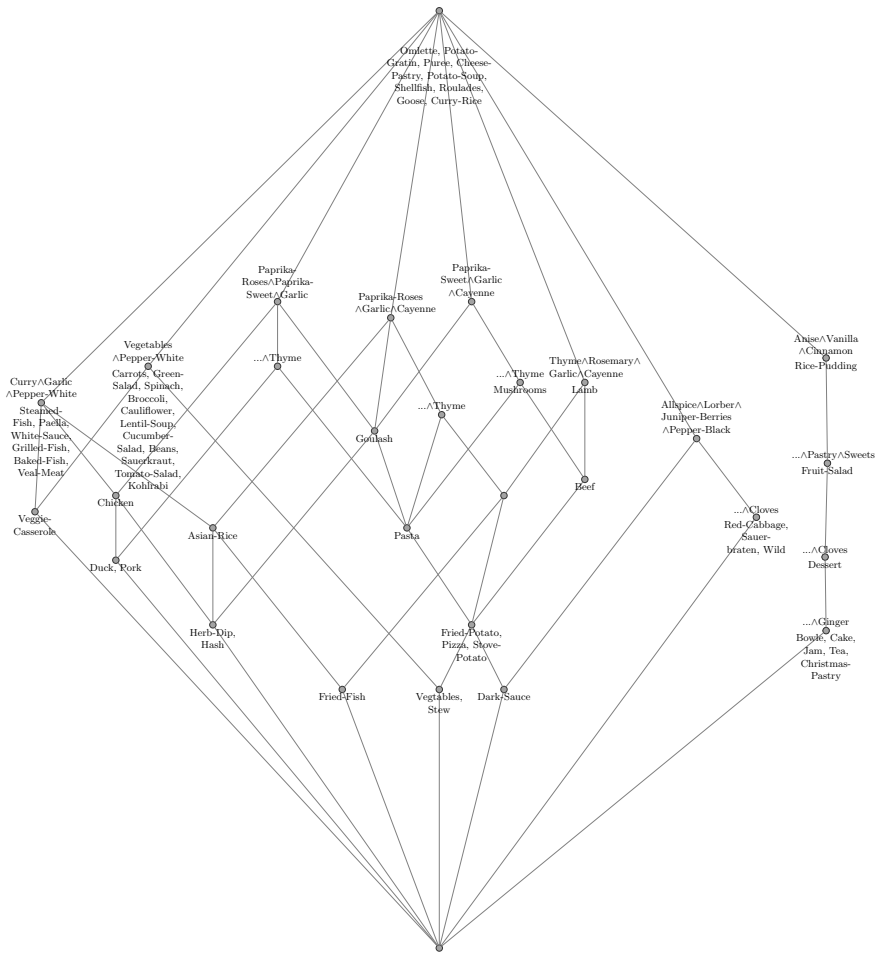

\section{Related Work}\label{sec:relate}
Measurement is an important field of study in many (scientific)
disciplines that involve the collection and analysis of
data. According to \citeauthor{stevens1946theory}
\cite{stevens1946theory} there are four feature categories that can be
measured, i.e., \emph{nominal}, \emph{ordinal}, \emph{interval} and
\emph{ratio} features. Although there are multiple extensions and
re-categorizations of the original four categories, e.g., most
recently~\citeauthor{Chrisman} introduced ten \cite{Chrisman}, for the
purpose of our work the original four suffice. Each of these
categories describe which operations are supported per feature
category. In the realm of formal concept analysis we work often with
\emph{nominal} and \emph{ordinal} features, supporting value
comparisons by $=$ and $<,>$. Hence grades of detail/membership
cannot be expressed. A framework to describe and analyze the
measurement for Boolean data sets has been introduced in
\cite{cmeasure} and \cite{scaling}, called \emph{scale-measures}. It
characterizes the measurement based on object clusters that are formed
according to common feature (attribute) value combinations. An
accompanied notion of dependency has been studied \cite{manydep},
which led to attribute selection based measurements of boolean
data. The formalism includes a notion of consistency enabling the
determination of different views and abstractions, called
\emph{scales}, to the data set. This approach is comparable to
\emph{OLAP}~\cite{olap} for databases, but on a conceptual
level. Similar to the feature dependency study is an approach for
selecting relevant attributes in contexts based on a mix of lattice
structural features and entropy
maximization~\cite{DBLP:conf/iccs/HanikaKS19}. All discussed
abstractions reduce the complexity of the data, making it easier to
understand by humans.

Despite the in this work demonstrated expressiveness of the
scale-measure framework, it is so far insufficiently studied in the
literature.  In particular algorithmical and practical calculation
approaches are missing. Comparable and popular machine learning
approaches, such as feature compressed techniques, e.g., \emph{Latent
  Semantic Analysis} \cite{lsa,lsaapp}, have the disadvantage that the
newly compressed features are not interpretable by means of the
original data and are not guaranteed to be consistent with said
original data. The methods presented in this paper do not have these
disadvantages, as they are based on meaningful and interpretable
features with respect to the original features using propositional
expressions. In particular preserving consistency, as we did, is not a
given, which was explicitly investigated in the realm scaling
many-valued formal contexts~\cite{logiscale} and implicitly studied for generalized
attributes~\cite{leonardOps}.

Earlier approaches to use scale contexts for complexity reduction in
data used constructs such as $(G_N\subseteq \mathcal{P}(N),N,\ni)$ for
a formal context $\context=(G,M,I)$ with $N\subseteq M$ and the
restriction that at least all intents of $\context$ restricted to $N$
are also intent in the scale~\cite{stumme99hierarchies}. Hence, the
size of the scale context concept lattice depends directly on the size
of the concept lattice of $\context$. This is particularly infeasible
if the number of intents is exponential, leading to incomprehensible
scale lattices. This is in contrast to the notion of scale-measures,
which cover at most the extents of the original context, and can
thereby display selected and interesting object dependencies of
scalable size.

\section{Conclusion}\label{conclusion}
With this work we have shed light on the hierarchy of
scale-measures. By applying multiple results from lattice theory,
especially concerning ideals, to said hierarchy, we were able to give
a more thorough structural description of
$\downarrow_{\mathfrak{F}_G}\Ext(\context)$. Our main theoretical
result is~\cref{prop:props}, which in turn leads to our practical
applications. In particular, based on this deeper understanding we
were able to present an algorithm for exploring the scale-hierarchy of
a contextual data set $\context$. Equipped with this algorithm a data
scaling expert may explore the lattice of scale-measures for a given
data set with respect to her preferences and the requirements of the
data analysis task. The practical evaluation and optimization of this
algorithm is a promising goal for future investigations. Even more
important, however, is the implementation and further development of
the automatic scaling framework, as outlined
in~\cref{sec:automatic}. This opens the door to empirical scale
recommendation studies and a novel approach for data preprocessing.
\printbibliography
\end{document}